\documentclass[twoside,english]{elsarticle}
\usepackage[T1]{fontenc}
\usepackage[latin9]{inputenc}
\usepackage{amsmath}
\usepackage{amsthm}
\usepackage{amssymb}

\makeatletter

\providecommand{\tabularnewline}{\\}

\theoremstyle{plain}
\newtheorem{thm}{\protect\theoremname}
\theoremstyle{definition}
\newtheorem{example}[thm]{\protect\examplename}
\theoremstyle{plain}
\newtheorem{prop}[thm]{\protect\propositionname}

\journal{TBD}

\makeatother

\usepackage{babel}
\providecommand{\examplename}{Example}
\providecommand{\propositionname}{Proposition}
\providecommand{\theoremname}{Theorem}

\begin{document}

\begin{frontmatter}{}

\title{Almost optimal manipulation of a pair\textbf{ }of alternatives}

\author[wms]{Jacek Szybowski}

\ead{jacek.szybowski@agh.edu.pl}

\author[dap]{Konrad Ku\l akowski\corref{cor1}}

\ead{konrad.kulakowski@agh.edu.pl}

\author[dap]{Sebastian Ernst}

\ead{sebastian.ernst@agh.edu.pl}

\cortext[cor1]{Corresponding author}

\address[wms]{AGH University of Science and Technology, The Faculty of Applied
Mathematics, al. Mickiewicza 30, 30-059 Krakow, Poland}

\address[dap]{AGH University of Science and Technology, The Department of Applied
Computer Science, Al. Adama Mickiewicza 30, 30-059 Krakow, Poland}
\begin{abstract}
The role of an expert in the decision-making process is crucial, as
the final recommendation depends on his disposition, clarity of mind,
experience, and knowledge of the problem. However, the recommendation
also depends on their honesty. But what if the expert is dishonest?
Then, the answer on how difficult it is to manipulate in a given case
becomes essential. In the presented work, we consider manipulation
of a ranking obtained by comparing alternatives in pairs. More specifically,
we propose an algorithm for finding an almost optimal way to swap
the positions of two selected alternatives. Thanks to this, it is
possible to determine how difficult such manipulation is in a given
case. Theoretical considerations are illustrated by a practical example.
\end{abstract}

\end{frontmatter}{}

\section{Introduction\label{sec:Introduction}}

The popularity of the pairwise comparison methods in the field of
multi-criteria decision analysis is largely due to their simplicity.
It is easier for a decision maker to compare two objects at the same
time, as opposed to comparing larger groups. Although the first systematic
use of pairwise comparison is attributed to Ramon Llull \citep{Colomer2011rlfa},
thirteenth-century alchemist and mathematician, it can be assumed
that also prehistoric people used this method in practice. In the
beginning, people were interested in qualitative comparisons. Over
time, however, this method gained a quantitative character. The twentieth-century
precursor of the quantitative use of pairwise comparison was Thurstone,
who harnessed this method to compare social values \citep{Thurstone1927tmop}.
Continuation of studies on the pairwise comparison method \citep{Thurstone1957trof,David1969tmop,Miller1966taow}
resulted the seminal work written by Saaty \citep{Saaty1977asmf}.
In this article Saaty proposed the Analytic Hierarchy Process (AHP)
\textendash{} a new multiple-criteria decision-making method. Thanks
to the popularity of AHP, the pairwise comparison method has become
one of the most frequently used decision-making techniques. Its numerous
variants and extensions have found application in economy \citep{Peterson1998evabt},
consumer research \citep{Gacula1984smif}, management \citep{Lepetu2012tuoa,Pecchia2013imcv,Urbaniec2022fseb},
construction \citep{Darko2019roao}, military science \citep{Hyde2012maot},
education and science \citep{Liberatore1997gdmi,Koczkodaj2014otqe},
chemical engineering \citep{Cui2020seoc}, oil industry \citep{Hu2020ofco}
and others. The method is constantly developing and inspires researchers
who conduct work on the inconsistency of the paired comparisons \citep{Bortot2023anpo,Bozoki2011albi,Brunelli2016saso,Koczkodaj2018aoii,Kulakowski2014tntb},
incompleteness of decision data \citep{Faramondi2019iahp,Csato2016ipcm,Kulakowski2020iifi,Kulakowski2019tqoi},
priority calculations \citep{Mazurek2022otdo,Bozoki2010ooco,kulakowski2020otgm,Kulakowski2021otsb,Janicki1996awoa},
representation of uncertain knowledge \citep{Yuan2023lsgd,Bartl2022anaf,PortodeLima2022nhaq}
as well as new methods based on the pairwise comparisons principle
\citep{PortodeLima2022nhaq,Rezaei2015bwmc,Koczkodaj2015pcs,Kulakowski2015hreg}.

Popularity of the decision-making methods also makes them vulnerable
to attacks and manipulations. This problem has been studied by several
researchers including Yager \citep{Yager2001pspm,Yager2002dasm} who
considered strategic preferential manipulations, Dong et al. \citep{Dong2021cras},
who addressed manipulation in the group decision-making or Sasaki
\citep{Sasaki2023smig}, on strategic manipulation in the context
of in group decisions with pairwise comparisons. Some aspects of decision
manipulation in the context of electoral systems are presented in
\citep{Gibbard1973movs,Gardenfors1976mosc,Taylor2005scat}.

In the presented work, we want to take a step towards determining
the degree of difficulty of manipulating in the pairwise comparison
method. For this purpose, we will propose an algorithm for calculating
the closest approximation of the pairwise comparison matrix (PCM),
which equates the priorities of two selected alternatives. We apply
a similar technique of orthogonal projections to that used in \citep{Koczkodaj1997aobf,Koczkodaj2020oopo}.
The difference between the initial matrix and the modified matrix
shows the degree of difficulty of a given manipulation. Although the
reasoning is done for additive matrices, the obtained result is also
valid for multiplicative matrices.

The article consists of four sections. \emph{Introduction} (Sec. \ref{sec:Introduction})
and \emph{Preliminaries} (Sec. \ref{sec:Preliminaries}) present the
state of research and introduce basic concepts and definitions in
the field of the pairwise comparison method. The third section, \emph{Towards
optimal manipulation of a pair of alternatives}, defines the procedure
to construct a manipulated pairwise comparisons matrix. It also contains
a method for determining the difficulty of manipulation. The article
ends with \emph{Conclusion} (Sec. \ref{sec:Conclusions}), summarizing
the achieved results.

\section{Preliminaries\label{sec:Preliminaries}}

\subsection{Multiplicative pairwise comparisons matrices}

Let us assume that we want perform pairwise comparison of a finite
set $E=\{e_{1},\ldots,e_{n}\}$ of alternatives. The comparisons can
be expressed in a \textit{pairwise comparisons matrix} (PCM) $M=[m_{ij}]$
with positive elements satisfying the \textit{reciprocity condition}
\begin{equation}
m_{ij}\cdot m_{ji}=1,
\end{equation}
for $i,j\in\{1,\ldots,n\}$.

Given a single PCM $M$, one of the more popular procedures to assign
a positive weight $w_{k}$ to each alternative $e_{k}$ ($k\in\{1,\ldots,n\}$),
showing its position in a ranking, is the Geometric Mean Method (GMM),
introduced in \citep{Crawford1985anot}. Then, the formula for $w_{k}$
can be calculated as the geometric mean of the $k$-th row elements:
\begin{equation}
w_{k}=\sqrt[n]{\prod_{j=1}^{n}m_{kj}}.\label{GM}
\end{equation}

If we want to standardize the resulting weight vector, we divide each
coordinate by the sum of all of them:

\[
\hat{w}_{k}=\frac{w_{k}}{\sum_{j=1}^{n}w_{j}}.
\]

Let us denote the set of all PCMs by $\mathcal{M}$.

\subsection{Additive pairwise comparisons matrices}

The family $\mathcal{M}$ is not a linear space. However, we can easily
transform every multiplicative PCM $M$ into an additive one using
the following map: 
\[
\varphi:\ \mathcal{M}\ni[m_{ij}]\mapsto[\ln(m_{ij})]\in\mathcal{A},
\]
where 
\[
\mathcal{A}:=\{[a_{ij}]:\ \forall i,j\in\{1,\ldots,n\}\ a_{ij}\in\mathbb{R}\textnormal{ and }a_{ij}+a_{ji}=0\},
\]
is a linear space of additive PCMs.

Obviously, we can define the map 
\[
\mu:\ \mathcal{A}\ni[a_{ij}]\mapsto[e^{a_{ij}}]\in\mathcal{M},
\]
such that 
\[
\mu\circ\varphi=id_{\mathcal{M}}
\]
and 
\[
\varphi\circ\mu=id_{\mathcal{A}}.
\]
Since $\varphi$ and $\mu$ are mutually reverse, from now on we will
consider only the additive case in order to use the algebraic structure
of $\mathcal{A}$.

If we treat an additive PCM $A$ as the image of $M\in\mathcal{M}$
by the map $\varphi$, we can also obtain the vector of weights $v$
by use of the logarithmic mapping: 
\[
v_{k}=\ln(w_{k}).
\]
By applying (\ref{GM}) we get 
\[
v_{k}=\ln\left(\sqrt[n]{\prod_{j=1}^{n}m_{kj}}\right)=\frac{\ln\left(\prod_{j=1}^{n}m_{kj}\right)}{n}=\frac{\sum_{j=1}^{n}\ln(m_{kj})}{n}=\frac{\sum_{j=1}^{n}a_{kj}}{n},
\]
so the $k$-th coordinate of $A$ can be calculated as the arithmetic
mean of the $k$-th row of $A$.

\section{Towards optimal manipulation of a pair of alternatives}

Let us start with a very simple example.
\begin{example}
\label{ex1} Consider a family of additive pairwise comparisons matrices
\[
A_{\varepsilon}=\left[\begin{array}{ccc}
0 & 1+\varepsilon & -1\\
-1-\varepsilon & 0 & 1\\
1 & -1 & 0
\end{array}\right].
\]
If we take $\varepsilon=\frac{1}{n}$ and $\varepsilon=-\frac{1}{n}$
($n\in\mathbb{N}$), we obtain two PCMs, whose weight vectors are
\[
v_{\frac{1}{n}}=\left(\frac{1}{3n},-\frac{1}{3n},0\right)^{T}
\]
and 
\[
v_{-\frac{1}{n}}=\left(-\frac{1}{3n},\frac{1}{3n},0\right)^{T},
\]
respectively.

It implies that the order of alternatives is $(a_{1},a_{3},a_{2})$
in the first case and $(a_{2},a_{3},a_{1})$ in the second case.

Since the standard Frobenious distance of the matrices is 
\[
||A_{\frac{1}{n}}-A_{-\frac{1}{n}}||=\sqrt{\left(\frac{2}{n}\right)^{2}+\left(\frac{2}{n}\right)^{2}}=\frac{2\sqrt{2}}{n},
\]
they can be arbitrarily close. 
\end{example}

Example \ref{ex1} shows that it is impossible to find the PCM closest
to a given one such that the ranking positions of two given alternatives
induced by both matrices will be reversed.

However, it is possible to find a matrix closest to a given one such
that their positions are the same.

\subsection{The tie spaces}

Fix $i,j\in\{1,\ldots,n\}$.

Let us define the subspace ${\cal A}_{ij}$ of all additive PCMs which
induce the ranking such that alternatives $i$ and $j$ are equal:

\[
{\cal A}_{ij}=\left\{ A\in{\cal A}:\frac{1}{n}\sum_{k=1}^{n}a_{ik}=\frac{1}{n}\sum_{k=1}^{n}a_{jk}\right\} .
\]

We will call such a linear space \textit{a tie space}.
\begin{prop}
\label{dim_aij} $\dim{\cal A}_{ij}=\frac{n^{2}-n}{2}-1$. 
\end{prop}

\begin{proof}
Each reciprocal additive matrix is uniquely defined by $\frac{n^{2}-n}{2}$
independent numbers $a_{qr},\ 1\leq q<r\leq n$ above the main diagonal.
The equation 
\begin{equation}
\sum_{p=1}^{n}a_{ip}=\sum_{p=1}^{n}a_{jp}\label{eq1}
\end{equation}
is equivalent to 
\[
\sum_{p=1}^{i-1}a_{ip}+0+\sum_{p=i+1}^{n}a_{ip}=\sum_{p=1}^{j-1}a_{jp}+0+\sum_{p=j+1}^{n}a_{jp}
\]
and to 
\[
-\sum_{p=1}^{i-1}a_{pi}+\sum_{p=i+1}^{n}a_{ip}=-\sum_{p=1}^{j-1}a_{pj}+\sum_{p=j+1}^{n}a_{jp}.
\]
Finally, (\ref{eq1}) is equivalent to 
\begin{equation}
a_{jn}=\sum_{p=1}^{j-1}a_{pj}-\sum_{p=1}^{i-1}a_{pi}+\sum_{p=i+1}^{n}a_{ip}-\sum_{p=j+1}^{n-1}a_{jp}.\label{eq2}
\end{equation}
This way we express $a_{jn}$ by the rest of the elements above the
main diagonal in the $i$-th and $j$-th row and column, which decreases
the dimension of the space by one.
\end{proof}
\noindent Now let us define a basis for the tie space ${\cal A}_{ij}$.
With no loss of generality we can assume that $i<j<n$. Let

\[
Z_{ij}:=\{(q,r):\ 1\leq q<r\leq n,\ \{q,r\}\cap\{i,j\}=\emptyset\}.
\]

\begin{prop}
The set $Z_{ij}$ has $\frac{(n-2)(n-3)}{2}$ elements. 
\end{prop}

\begin{proof}
The number of all elements above the main diagonal is $\frac{n^{2}-n}{2}$.
There are: 
\begin{itemize}
\item $i-1$ elements in the $i$-th column, 
\item $j-1$ elements in the $j$-th column, 
\item $n-i-1$ elements (excluding $a_{ij}$) in the $i$-th row, 
\item $n-j$ elements in the $j$-th row. 
\end{itemize}
Thus, the total number of $Z_{ij}$ elements equals 
\begin{eqnarray*}
\overline{\overline{Z_{ij}}} & = & \frac{n^{2}-n}{2}-(i-1)-(j-1)-(n-i-1)-(n-j)=\\
 & = & \frac{n^{2}-n}{2}-2n+3=\frac{n^{2}-n-4n+6}{2}=\frac{(n-2)(n-3)}{2}.
\end{eqnarray*}
\end{proof}
At first, for each $(q,r)\in Z_{ij}$ let us define $C^{qr}\in{\cal A}$,
whose elements are given by 
\[
c_{kl}^{qr}=\left\{ \begin{array}{rl}
1, & k=q,\ l=r\\
-1, & k=r,\ l=q\\
0, & \textnormal{otherwise}
\end{array}\right..
\]

Next, we define the elements of additive PCMs $D^{p}$ for $p\in\{1,\ldots,i-1\}$,
$E^{p}$ for $p\in\{1,\ldots,j-1\}$, $F^{p}$ for $p\in\{i+1,\ldots,j-1,j+1,\ldots,n\}$
and $G^{p}$ for $p\in\{j+1,\ldots,n-1\}$ by formulas:

\[
d_{kl}^{p}=\left\{ \begin{array}{rl}
1, & (k=p,\ l=i)\textnormal{ or }(k=n,\ l=j)\\
-1, & (k=i,\ l=p)\textnormal{ or }(k=j,\ l=n)\\
0, & \textnormal{otherwise}
\end{array}\right.,
\]

\[
e_{kl}^{p}=\left\{ \begin{array}{rl}
1, & (k=p,\ l=j)\textnormal{ or }(k=j,\ l=n,\ p\neq i)\\
-1, & (k=j,\ l=p)\textnormal{ or }(k=n,\ l=j,\ p\neq i)\\
2, & k=j,\ l=n,\ p=i\\
-2, & k=n,\ l=j,\ p=i\\
0, & \textnormal{otherwise}
\end{array}\right.,
\]

\[
f_{kl}^{p}=\left\{ \begin{array}{rl}
1, & (k=i,\ l=p)\textnormal{ or }(k=j,\ l=n)\\
-1, & (k=p,\ l=i)\textnormal{ or }(k=n,\ l=j)\\
0, & \textnormal{otherwise}
\end{array}\right.,
\]

and

\[
g_{kl}^{p}=\left\{ \begin{array}{rl}
1, & (k=j,\ l=p)\textnormal{ or }(k=n,\ l=j)\\
-1, & (k=p,\ l=j)\textnormal{ or }(k=j,\ l=n)\\
0, & \textnormal{otherwise}
\end{array}\right..
\]

\begin{prop}
\label{base_card} The total number of matrices $C^{q,r}$, $D^{p}$,
$E^{p}$, $F^{p}$ and $G^{p}$ is $\frac{n^{2}-n}{2}-1$. 
\end{prop}

\begin{proof}
Summing up the numbers of the consecutive matrices, we get the number
$\frac{(n-2)(n-3)}{2}+(i-1)+(j-1)+(n-i-1)+(n-j-1)==\frac{(n-2)(n-3)}{2}+2n-4=\frac{n^{2}-5n+6+4n-8}{2}=\frac{n^{2}-n}{2}-1.$
\end{proof}
\begin{thm}
A family of matrices 
\[
{\cal B}:=\{C^{qr}\}_{(q,r)\in Z_{ij}}\cup\{D^{p}\}_{p=1}^{i-1}\cup\{E^{p}\}_{p=1}^{j-1}\cup\{F^{p}\}_{p=i-1}^{j-1}\cup\{F^{p}\}_{p=j+1}^{n}\cup\{G^{p}\}_{p=j+1}^{n-1}
\]
is a basis of ${\cal A}_{ij}$.
\end{thm}

\begin{proof}
By Propositions \ref{dim_aij} and \ref{base_card}, the cardinality
of ${\cal B}$ is equal to the dimension of ${\cal A}_{ij}$, so it
is enough to show that each matrix $A\in{\cal A}_{ij}$ is generated
by matrices from ${\cal B}$.

For this purpose, let us define a matrix $H$ as a linear combination
of matrices from ${\cal B}$: 
\begin{eqnarray*}
H & := & \sum_{(q,r)\in Z_{ij}}a_{qr}C^{qr}+\sum_{p=1}^{i-1}a_{pi}D^{p}+\sum_{p=1}^{j-1}a_{pj}E^{p}+\sum_{p=i+1}^{j-1}a_{ip}F^{p}+\sum_{p=j+1}^{n}a_{ip}F^{p}+\\
 & + & \sum_{p=j+1}^{n-1}a_{jp}G^{p}.
\end{eqnarray*}
It is straightforward (as all but one addends of the sum are zeros)
that for $(q,r)\not\in\{(j,n),(n,j)\}$ we have 
\[
h_{qr}=a_{qr}.
\]

Likewise,

\begin{eqnarray*}
h_{jn} & = & \sum_{(q,r)\in Z_{ij}}a_{qr}\cdot0+\sum_{p=1}^{i-1}a_{pi}\cdot(-1)+\sum_{p=1}^{i-1}a_{pj}\cdot1+a_{ij}\cdot2+\sum_{p=i+1}^{j-1}a_{pj}\cdot1+\\
 & + & \sum_{p=i+1}^{j-1}a_{ip}\cdot1+\sum_{p=j+1}^{n}a_{ip}\cdot1+\sum_{p=j+1}^{n-1}a_{jp}\cdot(-1)=\\
 & = & \sum_{p=1}^{j-1}a_{pj}-\sum_{p=1}^{i-1}a_{pi}+\sum_{p=i+1}^{n}a_{ip}-\sum_{p=j+1}^{n-1}a_{jp}=a_{jn}.
\end{eqnarray*}

\noindent The last equality follows from (\ref{eq2}).

\noindent By analogy, $h_{nj}=a_{nj}$, so $A=H,$ which completes
the proof.
\end{proof}
\noindent Let us redefine ${\cal B}$ as a family of matrices 
\[
{\cal B}=\{B^{p}\}_{p=1}^{\frac{n^{2}-n}{2}-1}
\]
as follows:

We set $\{B^{p}\}_{p=1}^{\frac{(n-2)(n-3)}{2}}$ as matrices $\{C^{qr}\}_{(q,r)\in Z_{ij}}$,
ordered lexicografically, i.e. 
\begin{equation}
\begin{array}{l}
B^{1}:=C^{12},B^{2}:=C^{13},\ldots,B^{i-2}:=C^{1,i-1},B^{i-1}:=C^{1,i+1},\ldots,\\
B^{j-3}:=C^{1,j-1},B^{j-2}:=C^{1,j+1},\ldots,B^{n-3}:=C^{1,n},\ldots,B^{\frac{(n-2)(n-3)}{2}}:=C^{n-1,n}.
\end{array}\label{C}
\end{equation}

Next, we define 
\begin{eqnarray}
B^{\frac{(n-2)(n-3)}{2}+p}:=D^{p}, & p=1,\ldots,i-1,\label{D}\\
B^{\frac{(n-2)(n-3)}{2}+i-1+p}:=E^{p}, & p=1,\ldots,j-1,\label{E}\\
B^{\frac{(n-2)(n-3)}{2}+i+j-2+p}:=F^{p}, & p=i+1,\ldots,j-1,\label{F1}\\
B^{\frac{(n-2)(n-3)}{2}+j-3+p}:=F^{p}, & p=j+1,\ldots,n,\label{F2}\\
B^{\frac{(n-2)(n-3)}{2}+n-3+p}:=G^{p}, & p=j+1,\ldots,n-1.\label{G}
\end{eqnarray}

\begin{example}
\label{base} Consider $n=5,\ i=2$ and $j=3$.

\noindent Since $Z_{23}=\{(1,4),(1,5),(4,5)\}$, we get the following
basis of ${\cal A}_{23}$:\\

\noindent $B^{1}=C^{14}=\left(\begin{array}{ccccc}
0 & 0 & 0 & 1 & 0\\
0 & 0 & 0 & 0 & 0\\
0 & 0 & 0 & 0 & 0\\
-1 & 0 & 0 & 0 & 0\\
0 & 0 & 0 & 0 & 0
\end{array}\right),$ $B^{2}=C^{15}=\left(\begin{array}{ccccc}
0 & 0 & 0 & 0 & 1\\
0 & 0 & 0 & 0 & 0\\
0 & 0 & 0 & 0 & 0\\
0 & 0 & 0 & 0 & 0\\
-1 & 0 & 0 & 0 & 0
\end{array}\right),$

\noindent $B^{3}=C^{45}=\left(\begin{array}{ccccc}
0 & 0 & 0 & 0 & 0\\
0 & 0 & 0 & 0 & 0\\
0 & 0 & 0 & 0 & 0\\
0 & 0 & 0 & 0 & 1\\
0 & 0 & 0 & -1 & 0
\end{array}\right),$ $B^{4}=D^{1}=\left(\begin{array}{ccccc}
0 & 1 & 0 & 0 & 0\\
-1 & 0 & 0 & 0 & 0\\
0 & 0 & 0 & 0 & -1\\
0 & 0 & 0 & 0 & 0\\
0 & 0 & 1 & 0 & 0
\end{array}\right),$

\noindent $B^{5}=E^{1}=\left(\begin{array}{ccccc}
0 & 0 & 1 & 0 & 0\\
0 & 0 & 0 & 0 & 0\\
-1 & 0 & 0 & 0 & 1\\
0 & 0 & 0 & 0 & 0\\
0 & 0 & -1 & 0 & 0
\end{array}\right),$ $B^{6}=E^{2}=\left(\begin{array}{ccccc}
0 & 0 & 0 & 0 & 0\\
0 & 0 & 1 & 0 & 0\\
0 & -1 & 0 & 0 & 2\\
0 & 0 & 0 & 0 & 0\\
0 & 0 & -2 & 0 & 0
\end{array}\right),$

\noindent $B^{7}=F^{4}=\left(\begin{array}{ccccc}
0 & 0 & 0 & 0 & 0\\
0 & 0 & 0 & 1 & 0\\
0 & 0 & 0 & 0 & 1\\
0 & -1 & 0 & 0 & 0\\
0 & 0 & -1 & 0 & 0
\end{array}\right),$ $B^{8}=F^{5}=\left(\begin{array}{ccccc}
0 & 0 & 0 & 0 & 0\\
0 & 0 & 0 & 0 & 1\\
0 & 0 & 0 & 0 & 1\\
0 & 0 & 0 & 0 & 0\\
0 & -1 & -1 & 0 & 0
\end{array}\right),$

\noindent $B^{9}=G^{4}=\left(\begin{array}{ccccc}
0 & 0 & 0 & 0 & 0\\
0 & 0 & 0 & 0 & 0\\
0 & 0 & 0 & 1 & -1\\
0 & 0 & -1 & 0 & 0\\
0 & 0 & 1 & 0 & 0
\end{array}\right).$ 
\end{example}

\subsection{Orthogonalization}

It is a well-known fact of linear algebra that any basis of an inner
product space can be easily transformed into an orthogonal basis by
a standard Gram-Schmidt process.

In particular, if we apply that to the basis $B^{1},\ldots,B^{\frac{n^{2}-n}{2}-1}$
of the ${\cal A}_{ij}$ vector space equipped with a standard Frobenius
inner product $\langle\cdot,\cdot\rangle$ then we obtain a pairwise
orthogonal basis 
\[
H^{1},\ldots,H^{\frac{n^{2}-n}{2}-1}
\]
as follows: 
\begin{eqnarray*}
H^{1} & = & B^{1},\\
H^{2} & = & B^{2}-\frac{\langle H^{1},B^{2}\rangle}{\langle H^{1},H^{1}\rangle}H^{1},\\
H^{3} & = & B^{3}-\frac{\langle H^{1},B^{3}\rangle}{\langle H^{1},H^{1}\rangle}H^{1}-\frac{\langle H^{2},B^{3}\rangle}{\langle H^{2},H^{2}\rangle}H^{2},\\
\cdots & = & \cdots\\
H^{\frac{n^{2}-n}{2}-1} & = & B^{\frac{n^{2}-n}{2}-1}-\sum_{p=1}^{\frac{n^{2}-n}{2}-2}\frac{\langle H^{p},B^{\frac{n^{2}-n}{2}-1}\rangle}{\langle H^{p},H^{p}\rangle}H^{p}.
\end{eqnarray*}

\begin{example}
\label{ort-bas} Let us consider matrices $B^{1},\ldots,B^{9}$ from
Example \ref{base}. We will apply the Gram-Schmidt process to obtain
an orthogonal basis $H^{1},\ldots,H^{9}$ of ${\cal A}_{23}$:

\begin{eqnarray*}
H^{1}=B^{1},\\
\langle H^{1},B^{2}\rangle & = & 0\Rightarrow H^{2}=B^{2},\\
\langle H^{1},B^{3}\rangle & = & \langle H^{2},B^{3}\rangle=0\Rightarrow H^{3}=B^{3},\\
\langle H^{1},B^{4}\rangle & = & \langle H^{2},B^{4}\rangle=\langle H^{3},B^{4}\rangle=0\Rightarrow H^{4}=B^{4},\\
\langle H^{1},B^{5}\rangle & = & \langle H^{2},B^{5}\rangle=\langle H^{3},B^{5}\rangle=0,\ \langle H^{4},B^{5}\rangle=-2,\ \langle H^{4},H^{4}\rangle=4\Rightarrow\\
 & \Rightarrow & H^{5}=B^{5}+\frac{1}{2}H^{4}=\left(\begin{array}{ccccc}
0 & \frac{1}{2} & 1 & 0 & 0\\
-\frac{1}{2} & 0 & 0 & 0 & 0\\
-1 & 0 & 0 & 0 & \frac{1}{2}\\
0 & 0 & 0 & 0 & 0\\
0 & 0 & -\frac{1}{2} & 0 & 0
\end{array}\right),
\end{eqnarray*}

\begin{eqnarray*}
\langle H^{1},B^{6}\rangle & = & \langle H^{2},B^{6}\rangle=\langle H^{3},B^{6}\rangle=0,\ \langle H^{4},B^{6}\rangle=-4,\ \langle H^{5},B^{6}\rangle=2,\\
\langle H^{5},H^{5}\rangle & = & 3\Rightarrow H^{6}=B^{6}+H^{4}-\frac{2}{3}H^{5}=\left(\begin{array}{ccccc}
0 & \frac{2}{3} & -\frac{2}{3} & 0 & 0\\
-\frac{2}{3} & 0 & 1 & 0 & 0\\
\frac{2}{3} & -1 & 0 & 0 & \frac{2}{3}\\
0 & 0 & 0 & 0 & 0\\
0 & 0 & -\frac{2}{3} & 0 & 0
\end{array}\right),\\
\langle H^{1},B^{7}\rangle & = & \langle H^{2},B^{7}\rangle=\langle H^{3},B^{7}\rangle=0,\ \langle H^{4},B^{7}\rangle=-2,\ \langle H^{5},B^{7}\rangle=1,\\
\langle H^{6},B^{7}\rangle & = & \frac{4}{3},\ \langle H^{6},H^{6}\rangle=\frac{14}{3}\Rightarrow\\
 & \Rightarrow & H^{7}=B^{7}+\frac{1}{2}H^{4}-\frac{1}{3}H^{5}-\frac{2}{7}H^{6}=\left(\begin{array}{ccccc}
0 & \frac{1}{7} & -\frac{1}{7} & 0 & 0\\
-\frac{1}{7} & 0 & -\frac{2}{7} & 1 & 0\\
\frac{1}{7} & \frac{2}{7} & 0 & 0 & \frac{1}{7}\\
0 & -1 & 0 & 0 & 0\\
0 & 0 & -\frac{1}{7} & 0 & 0
\end{array}\right),
\end{eqnarray*}

\begin{eqnarray*}
\langle H^{1},B^{8}\rangle & = & \langle H^{2},B^{8}\rangle=\langle H^{3},B^{8}\rangle=0,\ \langle H^{4},B^{8}\rangle=-2,\ \langle H^{5},B^{8}\rangle=1,\\
\langle H^{6},B^{8}\rangle & = & \frac{4}{3},\ \langle H^{7},B^{8}\rangle=\frac{2}{7},\ \langle H^{7},H^{7}\rangle=\frac{16}{7}\Rightarrow\\
 & \Rightarrow & H^{8}=B^{8}+\frac{1}{2}H^{4}-\frac{1}{3}H^{5}-\frac{2}{7}H^{6}-\frac{1}{8}H^{7}=\\
 & = & \left(\begin{array}{ccccc}
0 & \frac{1}{8} & -\frac{1}{8} & 0 & 0\\
-\frac{1}{8} & 0 & -\frac{1}{4} & -\frac{1}{8} & 1\\
\frac{1}{8} & \frac{1}{4} & 0 & 0 & \frac{1}{8}\\
0 & \frac{1}{8} & 0 & 0 & 0\\
0 & -1 & -\frac{1}{8} & 0 & 0
\end{array}\right),
\end{eqnarray*}

\begin{eqnarray*}
\langle H^{1},B^{9}\rangle & = & \langle H^{2},B^{9}\rangle=\langle H^{3},B^{9}\rangle=0,\ \langle H^{4},B^{9}\rangle=2,\ \langle H^{5},B^{9}\rangle=-1,\\
\langle H^{6},B^{9}\rangle & = & -\frac{4}{3},\ \langle H^{7},B^{9}\rangle=-\frac{2}{7},\ \langle H^{8},B^{9}\rangle=-\frac{1}{4},\ \langle H^{8},H^{8}\rangle=\frac{9}{4}\Rightarrow\\
 & \Rightarrow & H^{9}=B^{9}-\frac{1}{2}H^{4}+\frac{1}{3}H^{5}+\frac{2}{7}H^{6}+\frac{1}{8}H^{7}+\frac{1}{9}H^{8}=\\
 & = & \left(\begin{array}{ccccc}
0 & -\frac{1}{9} & \frac{1}{9} & 0 & 0\\
\frac{1}{9} & 0 & \frac{2}{9} & \frac{1}{9} & \frac{1}{9}\\
-\frac{1}{9} & -\frac{2}{9} & 0 & 1 & -\frac{1}{9}\\
0 & -\frac{1}{9} & -1 & 0 & 0\\
0 & -\frac{1}{9} & \frac{1}{9} & 0 & 0
\end{array}\right).
\end{eqnarray*}
\end{example}

\subsection{The best approximation of a PCM equating two alternatives}

Consider an additive PCM $A$. In order to find its projection $A'$
onto the subspace ${\cal A}_{ij}$ we present $A'$ as a linear combination
of the orthogonal basis vectors 
\[
H^{1},\ldots,H^{\frac{n^{2}-n}{2}-1}.
\]
We will look for the factors 
\[
\varepsilon_{1},\ldots,\varepsilon_{\frac{n^{2}-n}{2}-1}
\]
such that $A'=\varepsilon_{1}H^{1}+\ldots\varepsilon_{\frac{n^{2}-n}{2}-1}H^{\frac{n^{2}-n}{2}-1}$.\\

Then, $\forall C\in{\cal A}_{ij},\ \langle A-A',C\rangle_{F}=0$,
which is equivalent to the system of linear equations:\\
 $\left\{ \begin{array}{l}
\langle A,H^{1}\rangle_{F}-\varepsilon_{1}\langle H^{1},H^{1}\rangle_{F}=0,\\
\langle A,H^{2}\rangle_{F}-\varepsilon_{2}\langle H^{2},H^{2}\rangle_{F}=0,\\
\cdots\\
\langle A,H^{\frac{n^{2}-n}{2}-1}\rangle_{F}-\varepsilon_{\frac{n^{2}-n}{2}-1}\langle H^{\frac{n^{2}-n}{2}-1},H^{\frac{n^{2}-n}{2}-1}\rangle_{F}=0,
\end{array}\right.$\\

Its solutions:

\[
\varepsilon_{k}=\frac{\langle A,H^{k}\rangle_{F}}{\langle H^{k},H^{k}\rangle_{F}},\ k=1,\ldots,\frac{n^{2}-n}{2}-1.
\]

Thus, the PCM $A'$ which generates a ranking equating the $i$-th
and $j$-th alternatives and which is the closest to $A$ can be calculated
from the formula

\begin{equation}
A'=\sum_{k=1}^{\frac{n^{2}-n}{2}-1}\frac{\langle A,H^{k}\rangle_{F}}{\langle H^{k},H^{k}\rangle_{F}}H^{k}.\label{Aprime}
\end{equation}

\begin{example}
\label{exa:example-8}Let us consider the PCM 
\begin{equation}
A=\left(\begin{array}{ccccc}
0 & -5 & 2 & 0 & 4\\
5 & 0 & 2 & 5 & -6\\
-2 & -2 & 0 & 4 & -9\\
0 & -5 & -4 & 0 & -8\\
-4 & 6 & 9 & 8 & 0
\end{array}\right).\label{eq:m-before-manipulation}
\end{equation}

The weights in a ranking vector obtained as the arithmetic means of
row elements of $A$ are 
\[
w=(0.2,1.2,-1.8,-3.4,3.8)^{T}.
\]

In order to find the PCM closest to $A$ which generates a ranking
equating the second and the third alternative, we take the orthogonal
basis $H^{1},\ldots,H^{9}$ described in Ex. \ref{ort-bas}. Next,
we calculate the coefficients in (\ref{Aprime}):

\medskip{}

{\footnotesize{}}%
\begin{tabular}{|c|c|c|c|c|c|c|c|c|c|}
\hline 
{\footnotesize{}$k$} & {\footnotesize{}$1$} & {\footnotesize{}$2$} & {\footnotesize{}$3$} & {\footnotesize{}$4$} & {\footnotesize{}$5$} & {\footnotesize{}$6$} & {\footnotesize{}$7$} & {\footnotesize{}$8$} & {\footnotesize{}$9$}\tabularnewline
\hline 
{\footnotesize{}$\langle A,H^{k}\rangle_{F}$} & {\footnotesize{}$0$} & {\footnotesize{}$8$} & {\footnotesize{}$-16$} & {\footnotesize{}$8$} & {\footnotesize{}$-10$} & {\footnotesize{}$-17.3333$} & {\footnotesize{}$4.285714$} & {\footnotesize{}$-18.25$} & {\footnotesize{}$12.22222$}\tabularnewline
\hline 
{\footnotesize{}$\langle H^{k},H^{k}\rangle_{F}$} & {\footnotesize{}$2$} & {\footnotesize{}$2$} & {\footnotesize{}$2$} & {\footnotesize{}$4$} & {\footnotesize{}$3$} & {\footnotesize{}$4.666666$} & {\footnotesize{}$2.285714$} & {\footnotesize{}$2.25$} & {\footnotesize{}$2.222222$}\tabularnewline
\hline 
{\footnotesize{}$\frac{\langle A,H^{k}\rangle_{F}}{\langle H^{k},H^{k}\rangle_{F}}$} & {\footnotesize{}$0$} & {\footnotesize{}$4$} & {\footnotesize{}$-8$} & {\footnotesize{}$2$} & {\footnotesize{}$-3.33333$} & {\footnotesize{}$-3.71429$} & {\footnotesize{}$1.875$} & {\footnotesize{}$-8.11111$} & {\footnotesize{}$5.5$}\tabularnewline
\hline 
\end{tabular}{\footnotesize{}. }{\footnotesize\par}

\medskip{}
\noindent Finally, we obtain the orthogonal projection of $A$ onto
${\cal A}_{23}$: 
\begin{equation}
A'=\left(\begin{array}{ccccc}
0 & -3.5 & -0.5 & 0 & 4\\
3.5 & 0 & -1 & 3.5 & -7.5\\
-0.5 & 1 & 0 & 5.5 & -7.5\\
0 & -3.5 & -5.5 & 0 & -8\\
-4 & 7.5 & 7.5 & 8 & 0
\end{array}\right).\label{eq:m-after-manipulation}
\end{equation}
The corresponding vector of weights is 
\[
w'=(0.2,-0.3,-0.3,-3.4,3.8)^{T}.
\]

Let us notice that the weights of the second and third alternatives
are actually equal. What's more, the weights of the rest of alternatives
have not changed. The common weight of the second and the third alternative
in $w'$ is the arithmetic mean of the corresponding weights in $w$.
\end{example}

It appears that the remark above is true regardless of the dimension
of the PCM and of the choice of the two alternatives whose weights
are equalized, i.e the following theorem is true:
\begin{thm}
Let $A=[a_{kl}]\in{\cal {A}}$, $i,j\in\{1,\ldots,n\}$, and $A'=[a'_{kl}]$
be the orthogonal projection of $A$ onto ${\cal A}_{ij}$. Then\\
 (1) For each $k\not\in\{i,j\}$ 
\begin{equation}
\sum_{l=1}^{n}a'_{kl}=\sum_{l=1}^{n}a_{kl},\label{eq_weights}
\end{equation}
(2) 
\begin{equation}
\sum_{l=1}^{n}a'_{il}=\sum_{l=1}^{n}a'_{jl}=\frac{\sum_{l=1}^{n}a_{il}+\sum_{l=1}^{n}a_{jl}}{2}.\label{am_weights}
\end{equation}
\end{thm}

\begin{proof}
Let us assume, without the loss of generality, that $i<j$.

Note that $(A'-A)\ \bot\ {\cal A}_{ij}$, which implies that $(A'-A)\ \bot\ B^{s}$
for $s=1\ldots,\frac{(n-2)(n-3)}{2}$, where $\{B^{s}\}$ is a base
of ${\cal A}_{ij}$ defined in (\ref{C})-(\ref{G}). Thus, for each
$B^{p}$ we can write the equality 
\[
\langle A'-A,B^{s}\rangle=0,
\]
which is equivalent to:\\
 {${\bf (1_{qr})}$} $a'_{qr}-a_{qr}=0,$ for $(q,r)\in Z_{ij}$,
$1\leq s\leq\frac{(n-2)(n-3)}{2}$;\\
 {${\bf (2_{p})}$} $a'_{pi}-a_{pi}-a'_{jn}+a_{jn}=0,$ for $p<i$,
$\frac{(n-2)(n-3)}{2}+1\leq s\leq\frac{(n-2)(n-3)}{2}+i-1$;\\
 {${\bf (3_{p})}$} $a'_{pj}-a_{pj}+a'_{jn}-a_{jn}=0,$ for $i\neq p<j$,
$\frac{(n-2)(n-3)}{2}+i\leq s\leq\frac{(n-2)(n-3)}{2}+\linebreak+i+j-2$
and $s\neq\frac{(n-2)(n-3)}{2}+2i-1$;\\
 {${\bf (3_{i})}$} $a'_{ij}-a_{ij}+2a'_{jn}-2a_{jn}=0,$ for $s=\frac{(n-2)(n-3)}{2}+2i-1$;\\
 {${\bf (4_{p})}$} $a'_{ip}-a_{ip}+a'_{jn}-a_{jn}=0,$ for $i<p\neq j$,
$\frac{(n-2)(n-3)}{2}+i+j-1\leq s\leq\linebreak\leq\frac{(n-2)(n-3)}{2}+n+j-2$;\\
 {${\bf (5_{p})}$} $a'_{jp}-a_{jp}-a'_{jn}+a_{jn}=0,$ for $p>j$,
$\frac{(n-2)(n-3)}{2}+n+j-1\leq s\leq\frac{n^{2}-n}{2}-1$.\\

Now, for the proof of (\ref{eq_weights}) assume that $k\not\in\{i,j\}$.
Then

$\begin{array}{lll}
S & := & {\displaystyle \sum_{l=1}^{n}a'_{kl}-\sum_{l=1}^{n}a_{kl}=\sum_{l=1}^{n}(a'_{kl}-a_{kl})=}\\
 & = & {\displaystyle \sum_{l<i}(a'_{kl}-a_{kl})+a'_{ki}-a_{ki}+\sum_{i<l<j}(a'_{kl}-a_{kl})+a'_{kj}-a_{kj}+}\\
 & + & {\displaystyle \sum_{l>j}(a'_{kl}-a_{kl})}.
\end{array}$.

From {${\bf (1_{kl})}$} for $l\not\in\{i,j\}$ we get 
\[
{\displaystyle \sum_{l<i}(a'_{kl}-a_{kl})+\sum_{i<l<j}(a'_{kl}-a_{kl})+\sum_{l>j}(a'_{kl}-a_{kl})=0},
\]
so 
\[
S=a'_{ki}-a_{ki}+a'_{kj}-a_{kj}.
\]

Consider three cases:\\
 (a) If $k<i$, we add equations {${\bf (2_{k})}$} and {${\bf (3_{k})}$}
and we get $S=0$.\\
 (b) If $i<k<j$, we subtract equation {${\bf (4_{k})}$} from {${\bf (3_{k})}$}
and we get $S=0$.\\
 (c) If $k>j$, we add equations {${\bf (4_{k})}$} and {${\bf (5_{k})}$}
and we get $S=0$.\\
 This concludes the proof of (\ref{eq_weights}).

\noindent Now, let us calculate

$\begin{array}{lll}
T & := & {\displaystyle \sum_{l=1}^{n}a'_{il}+\sum_{l=1}^{n}a'_{jl}-\sum_{l=1}^{n}a_{il}-\sum_{l=1}^{n}a_{jl}=}\\
 & = & {\displaystyle \sum_{l<i}(a'_{il}-a_{il})+\sum_{l<i}(a'_{jl}-a_{jl})}+a'_{ii}-a_{ii}+a'_{ji}-a_{ji}+\\
 & + & {\displaystyle \sum_{i<l<j}(a'_{il}-a_{il})+\sum_{i<l<j}(a'_{jl}-a_{jl})}+a'_{ij}-a_{ij}+a'_{jj}-a_{jj}+\\
 & + & {\displaystyle \sum_{l>j}(a'_{il}-a_{il})+\sum_{l>j}(a'_{jl}-a_{jl}).}
\end{array}$.

\noindent Since 
\[
a'_{ii}=a_{ii}=a'_{jj}=a_{jj}=a'_{ij}+a'_{ji}=a_{ij}+a_{ji}=0,
\]
it follows that

$\begin{array}{lll}
T & = & {\displaystyle \sum_{l<i}(a'_{il}-a_{il})+\sum_{l<i}(a'_{jl}-a_{jl})+\sum_{i<l<j}(a'_{il}-a_{il})+\sum_{i<l<j}(a'_{jl}-a_{jl})}+\\
 & + & {\displaystyle \sum_{l>j}(a'_{il}-a_{il})+\sum_{l>j}(a'_{jl}-a_{jl}).}
\end{array}$.

From {${\bf (2_{l})}$} and {${\bf (3_{l})}$} we get 
\begin{equation}
\sum_{l<i}(a'_{il}-a_{il})+\sum_{l<i}(a'_{jl}-a_{jl})=\sum_{l<i}(-a'_{jn}+a_{jn})+\sum_{l<i}(a'_{jn}-a_{jn})=0.\label{eq14}
\end{equation}

From {${\bf (4_{l})}$} and {${\bf (3_{l})}$} we get 
\begin{equation}
\sum_{i<l<j}(a'_{il}-a_{il})+\sum_{i<l<j}(a'_{jl}-a_{jl})=\sum_{i<l<j}(-a'_{jn}+a_{jn})+\sum_{i<l<j}(a'_{jn}-a_{jn})=0.\label{eq15}
\end{equation}

From {${\bf (4_{l})}$} and {${\bf (5_{l})}$} we get 
\begin{equation}
\sum_{l>j}(a'_{il}-a_{il})+\sum_{l>j}(a'_{jl}-a_{jl})=\sum_{l>j}(-a'_{jn}+a_{jn})+\sum_{l>j}(a'_{jn}-a_{jn})=0.\label{eq16}
\end{equation}

Equations (\ref{eq14}), (\ref{eq15}) and (\ref{eq16}) imply that
\[
T=0.
\]

On the other hand, $A'\in{\cal A}_{ij}$, which means that 
\[
\sum_{l=1}^{n}a'_{il}=\sum_{l=1}^{n}a'_{jl},
\]
so 
\[
T=2\sum_{l=1}^{n}a'_{il}-\sum_{l=1}^{n}a_{il}-\sum_{l=1}^{n}a_{jl}=2\sum_{l=1}^{n}a'_{jl}-\sum_{l=1}^{n}a_{il}-\sum_{l=1}^{n}a_{jl}=0,
\]
which proves (\ref{am_weights}). 
\end{proof}

\subsection{Measuring the ease of manipulation}

Manipulation carried out by an expert carries the risk of detection.
The consequences for the expert may vary, from loss of trust of the
person ordering the ranking to criminal sanctions. Of course, the
smaller the difference between the correct matrix and the manipulated
one, the lower the chances of dishonest answers being detected, and
thus the easier the manipulation. Therefore, the distance between
the individual elements of the matrix can be considered as the indicator
of the ease of manipulation.

It can be seen that both matrices differ in the rows and columns corresponding
to the replaced alternatives. For example, the absolute difference
of matrices $A$ (\ref{eq:m-before-manipulation}) and $A'$ (\ref{eq:m-after-manipulation})
given as $A''=|A-A'|=[|a_{ij}-a_{ij}^{'}|]$ is 
\[
A''=\left(\begin{array}{ccccc}
0 & 1.5 & 2.5 & 0 & 0\\
1.5 & 0 & 3 & 1.5 & 1.5\\
1.5 & 3 & 0 & 1.5 & 1.5\\
0 & 1.5 & 1.5 & 0 & 0\\
0 & 1.5 & 1.5 & 0 & 0
\end{array}\right).
\]
Thus, $A''$ has $14$ elements different than $0$. In general, such
a matrix can have $4n-6$ non-zero elements. Therefore, the Ease of
Manipulation Index (EMI) takes the form of the average difference
between elements of the original matrix and the manipulated one, i.e.

\[
\textit{EMI}(A,A')=\frac{1}{4n-6}\sum_{i,j=1}^{n}\left|a_{ij}-a_{ij}^{'}\right|.
\]
In the case of $A$ (\ref{eq:m-before-manipulation}) and $A'$ (\ref{eq:m-after-manipulation}),
the value of $\textit{EMI}(A,A')=1.785$.

\section{Conclusion and summary\label{sec:Conclusions}}

In the presented work, we introduce a method to find a closest approximation
of a PCM which equates two given alternatives (let's say, the $i$-th
and the $j$-th one). We also prove that the weights of all of the
other alternatives do not change, while the new weights of the equated
alternatives are equal to the arithmetic mean of the original ones.

Example \ref{ex1} shows that it is impossible to find the best approximation
of a PCM such that the positions of the $i$-th and the $j$-th alternatives
in a ranking are reversed. However, if two alternatives have the same
ranking, we may slightly change the element $a_{ij}$ in order to
tip the scales of victory in favor of one of them. The resulting matrix
will satisfy the manipulation condition.

We also proposed an Ease of Manipulation Index (EMI), which allows
to determine the average difficulty of switching the positions of
two alternatives. The analysis of $A''$ resulting from the difference
between the original and manipulated PCMs can also indicate pairwise
comparisons where manipulation is particularly difficult (the distance
between $a_{ij}$ and its counterpart $a_{ij}^{'}$ is exceptionally
large).

An alternative gradient method presented in \citep{Magnot2023agmf}
is a possible subject of the future research. Another possible generalization
could be incomplete PCMs considered for example in \citep{Kulakowski2019tqoi}.

\section*{Acknowledgments}

The research has been supported by The National Science Centre, Poland,
project no. 2021/41/B/HS4/03475 and by the AGH University of Science
and Technology (task no. 11.11.420.004).

\bibliographystyle{plain}
\bibliography{papers_biblio_reviewed}

\end{document}